\documentclass{article} 
\usepackage[preprint]{neurips_2024}

\usepackage[utf8]{inputenc} 
\usepackage[T1]{fontenc}    

\usepackage[dvipsnames]{xcolor}
\usepackage[colorlinks=true, linkcolor=black, citecolor=Red, urlcolor=blue]{hyperref}

\usepackage{amsmath}
\usepackage{amssymb}
\usepackage{amsfonts} 
\usepackage{amsthm}
\usepackage{mathtools}

\usepackage{booktabs}  
\usepackage{makecell}  
\usepackage{colortbl}  
\usepackage{multirow}  
\usepackage{multicol}  
\usepackage{subcaption} 

\usepackage{natbib}

\usepackage{algorithm}
\usepackage{algorithmic}

\newtheorem{lem}{Lemma}

\newtheorem{cor}{Corollary}
\newtheorem{defn}{Definition} 
\newtheorem{rem}{Remark}

\usepackage[shortlabels]{enumitem}  
\usepackage{pifont}  
\usepackage{fontawesome5} 
\usepackage[symbol]{footmisc} 

\usepackage{tcolorbox}
\tcbuselibrary{breakable}

\usepackage{array}  
\usepackage{tocloft} 

\definecolor{customblue}{HTML}{b9dcb6}
\definecolor{customblue2}{HTML}{a8dde1}

\usepackage{amsmath,amsfonts,bm}

\def\eqref#1{equation~\ref{#1}}

\def\1{\bm{1}}

\DeclareMathAlphabet{\mathsfit}{\encodingdefault}{\sfdefault}{m}{sl}
\SetMathAlphabet{\mathsfit}{bold}{\encodingdefault}{\sfdefault}{bx}{n}

\title{Graph Pseudotime Analysis and Neural Stochastic Differential Equations for Analyzing Retinal Degeneration Dynamics and Beyond}

\author{%
  Dai Shi,\thanks{Equal contribution.  \faEnvelope \ \texttt{dai.shi@sydney.edu.au}.\ To QuanQuan \faCat, a gift in solemn tribute, for thy love unyielding and thy company unfaltering; in memory’s hallowed halls, thou shalt dwell forevermore.}  \
  \ Kuan Yan,$^*$  
  \ Lequan Lin,  
  \ Yue Zeng,  
  \ Ting Zhang,  \\~\\
  \textbf{Dmytro Matsypura,  
  \ Mark C. Gillies,  
  \ Ling Zhu\thanks{Corresponding author: \texttt{ling.zhu@sydney.edu.au}.},  
  \ and Junbin Gao\thanks{Corresponding author: \texttt{junbin.gao@sydney.edu.au}.}} \\~\\
  \textit{University of Sydney, Australia}
}

\date{}

\begin{document}

\maketitle

\begin{abstract}
Understanding disease progression at the molecular pathway level usually requires capturing both structural dependencies between pathways and the temporal dynamics of disease evolution. In this work, we solve the former challenge by developing a biologically informed graph-forming method to efficiently construct pathway graphs for subjects from our newly curated JR5558 mouse transcriptomics dataset. We then develop Graph-level Pseudotime Analysis (GPA) to infer graph-level trajectories that reveal how disease progresses at the population level, rather than in individual subjects. Based on the trajectories estimated by GPA, we identify the most sensitive pathways that drive disease stage transitions. In addition, we measure changes in pathway features using neural stochastic differential equations (SDEs), which enables us to formally define and compute pathway stability and disease bifurcation points (points of no return), two fundamental problems in disease progression research. We further extend our theory to the case when pathways can interact with each other, enabling a more comprehensive and multi-faceted characterization of disease phenotypes. The comprehensive experimental results demonstrate the effectiveness of our framework in reconstructing the dynamics of the pathway, identifying critical transitions, and providing novel insights into the mechanistic understanding of disease evolution.


\end{abstract}

\section{Introduction}
Understanding disease progression through the lens of molecular pathways is crucial for identifying critical transitions and potential intervention targets \citep{liu2013pathway}. Although recent advances in time-series forecasting \citep{tripto2020evaluation}, differential equations \citep{xie2024rna}, and dynamical simulations \citep{xu2024equivariant, satorras2021n} have significantly improved our ability to quantify disease evolution, their applicability remains constrained when relevant biological signals cannot be consistently measured within the same individual. For example, single-cell transcriptomic profiling provides crucial insights into disease states, but requires cell lysis, making longitudinal observations within the same subject infeasible \citep{chen2022live}. This limitation presents a fundamental challenge in modeling the path-level disease dynamics and tasks such as identifying the key path-sensitive pathways (SP), the stability of the pathway \citep{khasminskii2012stochastic}, and the estimation of the bifurcation of the disease \citep{flores2023bifurcation}. In this work, we solve these challenges by developing Graph Pseudotime Analysis (GPA), which extends the traditional pseudotime analysis \citep{qiu2017single} to the object's pathway graph profiles formed by our biology-informed graph-forming algorithm. The GPA estimated trajectories among graphs offer us insights into identifying the SPs that are likely to induce the changes of disease stages. Furthermore, we model these trajectories using neural partial stochastic differential equations (SDEs) \citep{li2020scalable}, which provide a principled approach for analyzing pathway stability and detecting disease bifurcation points. Our comprehensive analysis and experimental results establish a novel paradigm for studying the genomic dynamics of pathway profiles, offering new insights into disease progression and potential therapeutic interventions. We summarize our contributions as follows. 

\begin{itemize}
    \item We provide a novel mouse dataset (namely JR5558) that integrates transcriptomics data for studying pathway-level disease dynamics. Furthermore, we form our research objects molecular pathway graphs based in a biology-informed manner. We then identify the SPs from our data by the graph regression techniques. 

    \item We develop GPA to uncover the orders between graphs, and we apply temporal GNN models to show the SPs that drive the stage change of the disease. 

   \item We further analyze the estimated pathway trajectories using SDEs, rigorously defining pathway-level dynamic stability and establishing conditions for disease bifurcation. By leveraging our learned SDEs, we precisely identify the timing of disease bifurcations and assess the stability of pathways. Furthermore, we extend our SDE framework to incorporate interactions of pathway, enabling a more comprehensive and multi-faceted characterization of disease phenotypes (Appendix \ref{append:sde}).

    \item We conduct comprehensive numerical experiments based on our developed approaches above and provide biological interpretations of our findings and suggestions for future studies. 
\end{itemize}

\begin{figure}[t]
\centering
\includegraphics[width=1\textwidth]{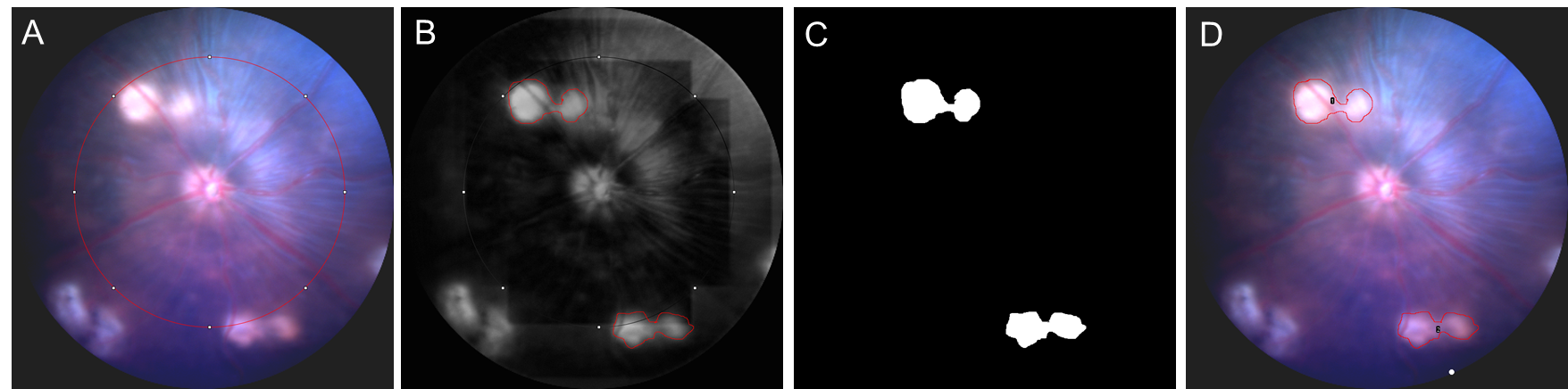} 
\caption{Illustration of the key steps of the quantification method. 
    (A) A circular region of interest (ROI) with a 283$\mu m$ radius centered on the optic nerve head was defined. (B) The image was converted to 8-bit grayscale, the background was removed using a 200$\mu m$ rolling ball algorithm, and lesions within the ROI were manually outlined using the freehand selection tool. (C) A threshold was applied to accurately segment the lesion areas. (D) The selected lesions were verified in the original color fundus image.}
\label{fig:02_fundus}
\end{figure}

\section{Preliminaries}\label{sec:preliminaries}
\paragraph{Notations} We leave the summary of the related works in Appendix \ref{append:related_works} but introduce the necessary notation as follows. Throughout this paper, we let $\mathcal G (\mathcal V, \mathcal E, \mathbf W)$ be the weighted graph with $\mathcal V$ and $\mathcal E$ be the sets of nodes and edges, and $\mathbf W \in \mathbb R^{|\mathcal V| \times |\mathcal V|}$ be the matrix contains all the edge weights with $w_{i,j} > 0$ if $i \sim j$. We also denote $\mathbf A$ and $\mathbf L \in \mathbb R^{N\times N}$ be the adjacency and Laplacian matrices of the graph, respectively, where $N$ is the number of nodes. Furthermore, we let $\{ (\lambda_i, \mathbf u_i) \}_{i=1}^N$ be the set of eigen-pairs of ${\mathbf L}$ where $\mathbf u_i$ are the row vectors of $\mathbf U$ obtained from the eigendecomposition of $\mathbf L = \mathbf U \boldsymbol{\Lambda} \mathbf U^\top$. Lastly, the feature matrix of the graph is denoted as $\mathbf X \in \mathbb R^{N\times d}$, where $d$ standards for the feature dimension. 

\paragraph{JR5558 Dataset}
We present a concise overview of the JR5558 dataset to establish a clear foundation for the analyses that follow. Specifically, the JR5558 mouse model, obtained from Jackson Laboratory (JAX stock \#005558), is a well-established system for studying neovascular age-related macular degeneration (nAMD). This study utilized bulk RNA-seq data collected from the retinas of 23 eight-week-old male JR5558 mice. RNA extraction was performed using the GenEluteTM Single Cell RNA Purification Kit (Sigma Aldrich, RNB300), with library preparation, quality control, and sequencing conducted by Novogene. 

The RNA-seq dataset contains expression levels for 56,748 genes, normalized as Fragments Per Kilobase of transcript per Million mapped reads (FPKM), of which 24,888 genes with corresponding Entrez IDs were used for pathway analysis. In addition to the RNA-seq data, fundus photographs were taken to quantify subretinal lesion severity. The lesion area was calculated as a percentage of the total retinal area using an automated image analysis pipeline. The key steps include delineating the optic nerve region, adjusting the image threshold, and validating the selected lesion area as illustrated in Figure~\ref{fig:02_fundus}. This methodology ensures consistent and accurate quantification of lesion severity. 


\section{Uncover The Disease Sensitive Genetic Pathways}
In this section, we show in detail how our graph-based methods can efficiently extract the key sensitive pathways (SP) via JR5558. This involves identifying the most sensitive pathways when we treat the disease development as a whole, i.e., there is no temporal (sequential)
relationship between objects. On the other hand, we also aim to uncover those SPs that drive the stage changes of the disease by building a pseudotime among objects' graph profiles. 
Figure.~\ref{fig:graph_forming_sensitivity} provides an illustration of how our model works.

\subsection{Biology-informed Graph Forming}\label{sec:graph_forming}
At the initial stage, we aim to form graphs that represent the pathway profile of each mouse. Specifically, we group genes (total 24,888) into 343 canonical molecular pathways by developing a web scraper extract detailed biological pathway information related to Mus musculus from the KEGG database\footnote{\url{https://www.genome.jp/kegg-bin/show_organism?menu_type=pathway_maps&org=mmu}}. We note that there are 9703 unique genes (out of 24,888) selected for forming genetic pathways, suggesting a large reduction in the complexity of the subsequent computations. The resulting genetic pathways are with different gene lengths, ranging from 2 to 1168. For each mouse in our dataset, we form a pathway graph with 343 nodes, i.e., genetic pathways as nodes\footnote{In the sequel, we will use pathways and nodes interchangeably.}. As the number of genes is different between pathways, and there are 9703 unique genes involved in forming pathways, for the sake of the feature propagation in the subsequent GNN models, we initially set the dimension of the feature as 9703. Node features are directly obtained from the FPKM (see the introduction of JR5558 dataset in Section \ref{sec:preliminaries}) with the rest part (i.e., the part does not match 9703 unique genes) as zeros. Furthermore, we build edges between nodes if gene overlap(s) exists between two nodes, and the edge weight is computed by the Euclidean distance between node features. Accordingly, for each mouse, the node feature matrix is $\mathbf X \in \mathbb R^{343 \times 9703}$, and the graph connectivity is the same for all mice but with different edge weights. Lastly, we formed 23 graphs for these mice based on the aforementioned paradigm. We finally note that it is commonly seen in the recent literature that edge weights are computed via the distance between node features such as protein graphs \citep{xu2024equivariant,wang2023graph} and electrical networks \citep{black2023understanding}.

\begin{figure}[t]
\centering
\includegraphics[width=1\textwidth]{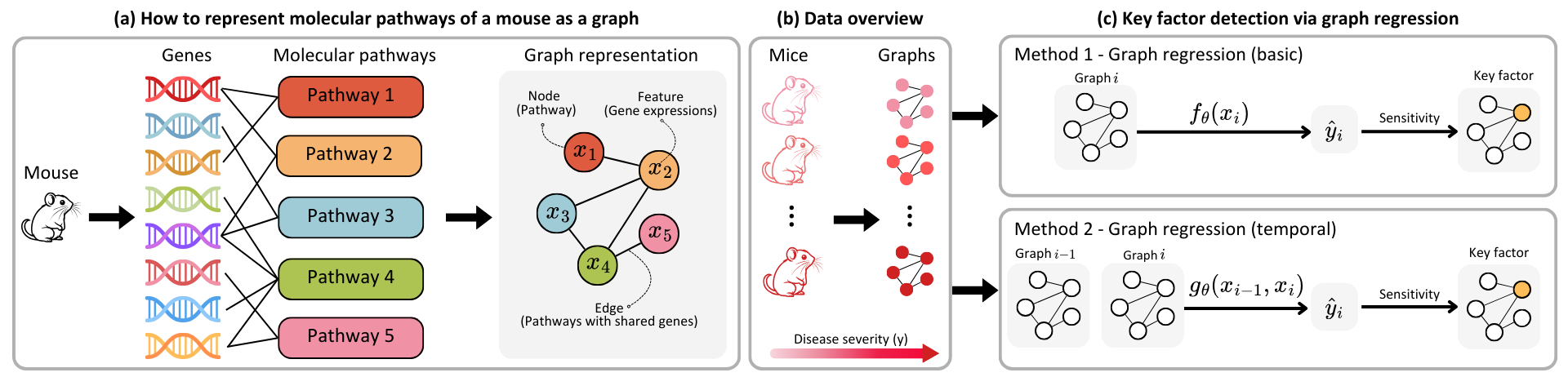} 
\caption{Illustration on how the pathway graphs are formed and how SPs are identified through two types of graph regressions.}
\label{fig:graph_forming_sensitivity}
\end{figure}

\subsection{Graph Regression and Detection of Disease Sensitive Pathways}\label{sec:key_factor_1}
Graph-level regression leverages graph adjacency and node feature information to approximate the target graph labels usually denoted as a continuous variable as $\mathbf y \in \mathbb R^m$, where $m$ is the number of graphs. The process can be formulated as 
\begin{align}
    f_{\boldsymbol{\theta}} (\mathbf A (i), \mathbf X(i)) \approx \mathbf y_i, \quad i \in [1,m],
\end{align}
where we let $f_{\boldsymbol{\theta}}$ be the learned graph regression model, usually containing GNN models with additional pooling and regression layers. Please check the details of the model implementation via Section \ref{sec:experiment} and Appendix \ref{append:experiment}. We note that the learned parameters (i.e., $\boldsymbol{\theta}$) are shared over all graphs. After the model is trained, we compute the model's sensitivity to all nodes. Specifically, for each node $\mathbf x_i$, we compute is sensitivity score as
\begin{align}\label{eq:sensitivity}
    \mathrm{Sensitivity}(\mathbf x_s) = \sum_{i=1}^m\left\|\frac{\partial \widehat {\mathbf y}_i}{\partial \mathbf x_s}\right \|,
\end{align}
where we denote $\widehat {\mathbf y}_i$ as the model prediction based on the input graph $i$ with $\mathbf A(i)$ and $\mathbf X(i)$. As there are $m$ graphs with different adjacency weights, one shall have $m$ number of the Jacobian norms (i.e., $\|\frac{\partial \widehat {\mathbf y}_i}{\partial \mathbf x_s}\|$) and we denote the sum of these Jacobian norms of node $s$ be the model's sensitivity on $s$. We remark that the node Jacobian norm has been widely applied to various GNN research, especially for analyzing the so-called over-squashing problem, in which nodes gradually lose their dependencies after a long-range propagation via GNNs, even though the graph is connected \citep{shi2023exposition}. Furthermore, the sensitivity scores denoted in Equation~(\ref{eq:sensitivity}) can be easily obtained via the Pytorch loss backward gradient once the model is trained, demonstrating the effectiveness of leveraging such definition on finding the SPs 
of the disease. 

\subsection{Discovering Order Amidst Chaos: Graph Pseudotime Analysis}
While our initial graph regression approach efficiently identifies key SPs, practical applications often demand the establishment of graph orders to enable graph-level time series forecasting. More importantly, constructing an ordered sequence of graphs allows for more meaningful clustering analyses, uncovering a biologically significant question: 

\begin{tcolorbox}[colback=cyan!5, breakable]
{\begin{center}
    \textit{What are the pathways driving transitions between distinct stages of the disease?}
\end{center}
    
    }
\end{tcolorbox}
To resolve this problem, we novelly developed graph-level pseudotime analysis (GPA) adopted from the Monocle model and its variants developed in the recent work in \citep{haghverdi2016diffusion,qiu2017single}. We note that our principal purpose is to reveal the temporal orders of pathway graphs rather than estimating trajectories (dynamics) of the single mRNA, which serves as the main objective of the pseudotime methods \citep{prabhakaran2016dirichlet,gut2015trajectories}. Thus, we leave the discussion between different pseudotime methods and their potential extension to graphs as future works.

\paragraph{Graph Dimensionality Reduction via Positional Encoding}
Our GPA starts with the dimensionality reduction of the graphs. First, we conduct the positional encoding to enrich the graph feature information with their spectral characteristics. That is for each node feature matrix $\mathbf X(i)$ from the graph $\mathbf A(i)$, we have 
\begin{align}
    \widehat{\mathbf X}(i) = \mathrm{PE}(\mathbf X(i), \mathbf A(i)) = [\mathbf X(i) \,\, || \,\, \mathbf U(i)],
\end{align}
where $[\cdot || \cdot]$ denotes the concatenation of the vectors, and $\mathbf U(i)$ is the eigenvector matrix of graph $i$. We then conduct mean and maximum pooling operations to ensure the final output vector (denoted as $\widehat{\mathbf x}(i)$) is sufficiently representative in terms of both node feature and graph topology information. That is  
\begin{align}
    \widehat{\mathbf x}(i) = [\mathrm{MaxPool}(\widehat{\mathbf X}(i)) \,\, || \,\, \mathrm{MeanPool}(\widehat{\mathbf X}(i))].
\end{align}
One can check that for each graph $i$, its $\widehat{\mathbf x}(i)$ is with the size $\mathbb R^{N+d}$. Then, their low-dimensional representations are obtained through UMAP \citep{mcinnes2018umap}, that is $\widehat{\mathbf z}(i) = \mathrm{UMAP}(\widehat{\mathbf x}(i))$. In practice, we usually let the dimension of $\widehat{\mathbf z}$ to be 2, and to sufficiently leverage the disease severity information \citep{shi2024design}, we further concatenate each $\widehat{\mathbf z} (i)$ with its $\mathbf y_i$, i.e., $\mathbf z(i) = [\widehat{\mathbf z}(i) || \mathbf y_i]$. As a result, we obtained three-dimensional representations of all the graphs as shown in Figure~\ref{fig:clustering}. We note that although there are many graph-level dimensionality reduction methods, they are often trainable and designed by fitting various downstream tasks, such as link prediction \citep{kipf2016variational} and causality analysis \citep{ng2019graph}. As our goal here is to efficiently obtain graph representations, we select those widely applied unsupervised methods (i.e., UMAP) as our tool and leave the discussion on obtaining the optimized graph representation in future works. 

\begin{figure}[t]
    \centering
    \scalebox{1}{
    \begin{subfigure}{0.32\linewidth}  
        \centering
        \includegraphics[width=\linewidth]{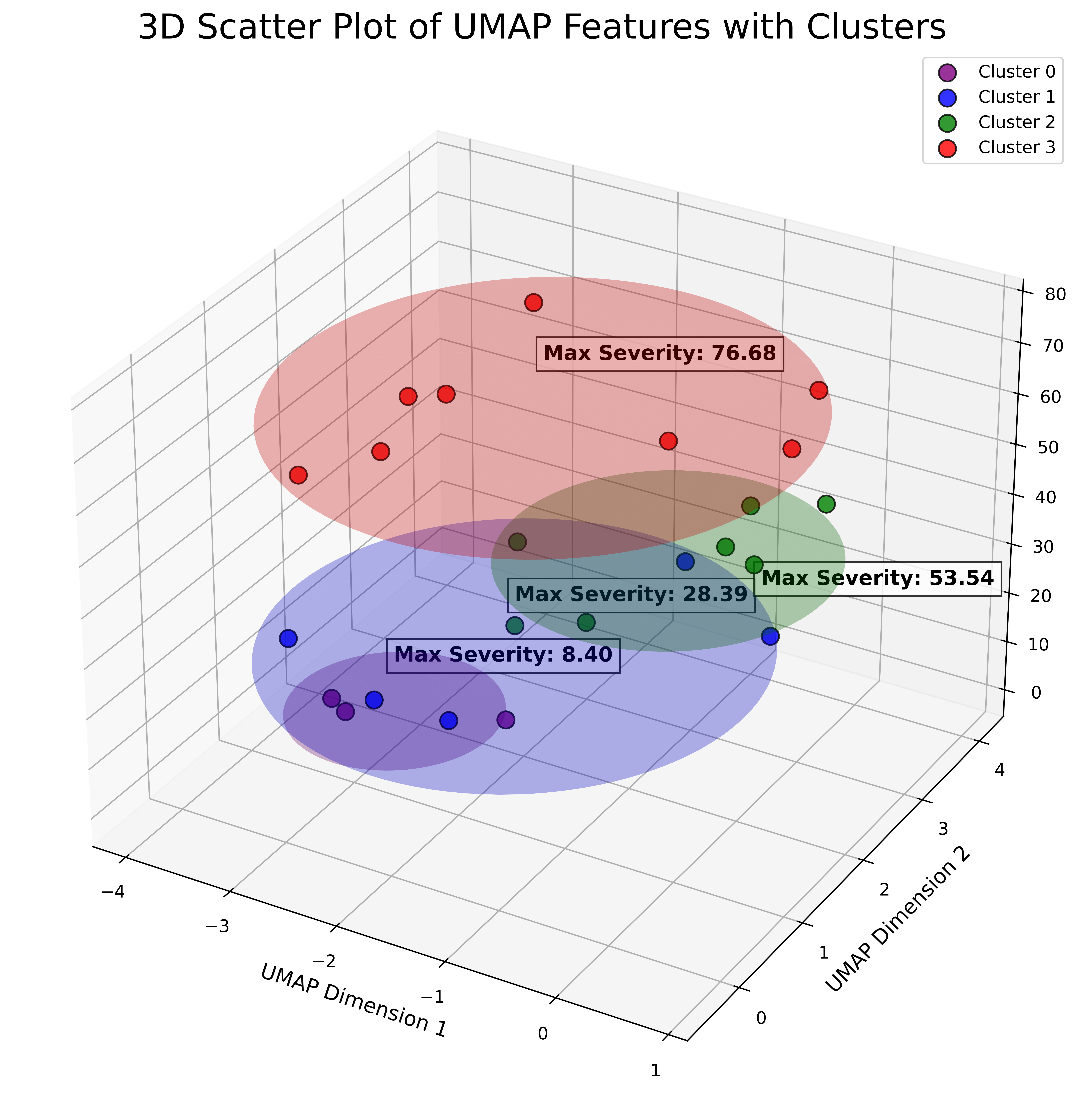}
        \caption{}
        \label{fig:clustering}
    \end{subfigure}
    \begin{subfigure}{0.32\linewidth}
        \centering
        \includegraphics[width=\linewidth]{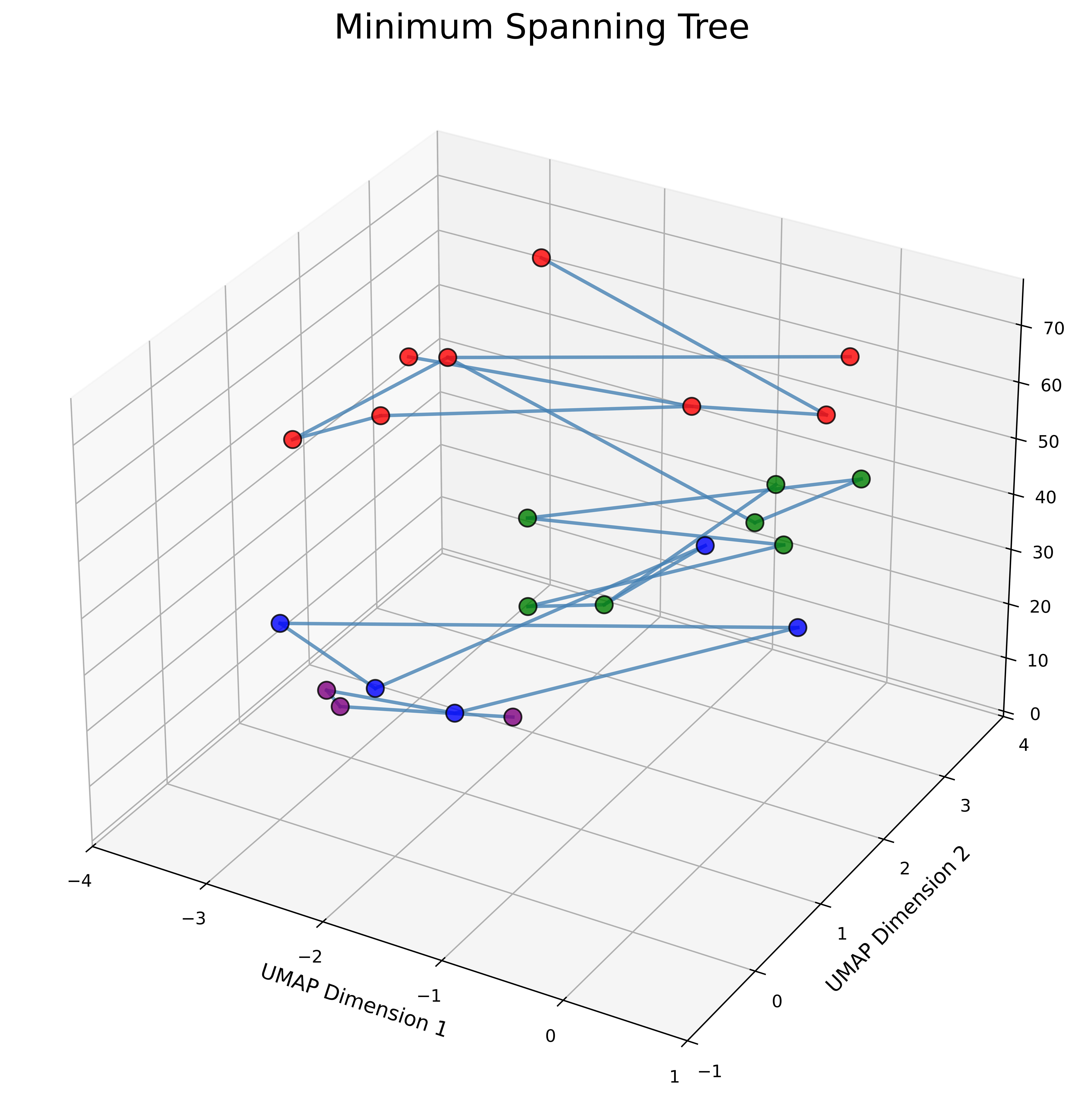}
        \caption{}
        \label{fig:mst}
    \end{subfigure}
    \begin{subfigure}{0.32\linewidth}  
        \centering
        \includegraphics[width=\linewidth]{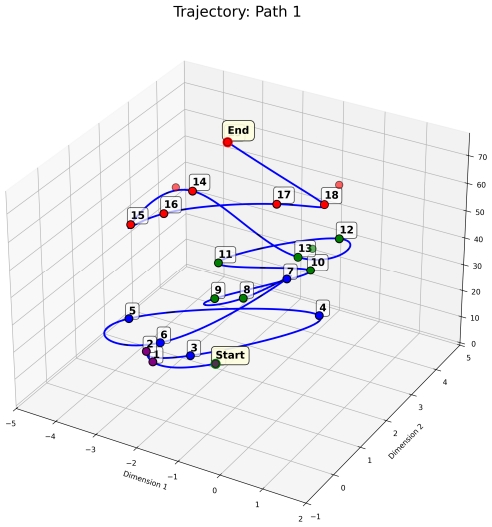}
        \caption{}
        \label{fig:estimated_trajectories}
    \end{subfigure}
    }
    \caption{Process of generating graph trajectory via GPA. (a): Graph dimensionality reduction and clustering; (b): Minimum spanning tree; (c): Trajectory estimation using shortest distance method.}
    \label{fig:trajectory_est}
\end{figure}

\paragraph{Minimum Spanning Tree and Trajectory Estimation}
After obtaining the low-dimensional representations of the graphs, we aim to establish an intrinsic ordering among them. To achieve this, we first construct a K-nearest neighbors (KNN) graph over the embedded representations $\widehat{\mathbf z}$, capturing local similarities in the latent space. We then apply a Minimum Spanning Tree (MST) to refine this structure, ensuring global connectivity while minimizing the total edge weight. This process effectively restructures the KNN graph into a hierarchical framework that preserves essential topological relationships. Biologically, this formulation shifts the focus from \textit{modeling disease progression at an individual level to capturing its systemic evolution across an entire species}, providing a more holistic perspective on disease dynamics. Finally, by fixing pairs of starting and ending objects, e.g., objects with the lowest/highest disease severity, the disease evolution trajectory can be obtained by leveraging the shortest path distance algorithm, which provides an estimate in which every node on the path will travel to the next node with shortest distance (e.g., edge weights). We note that given different starting and ending points for one dataset, the estimated trajectory can be many, suggesting a transaction between distributions, i.e., from the collection (distribution) of all starting points to the ending points. Figure~\ref{fig:trajectory_est} illustrates our trajectory estimation process.

\subsection{Temporal GCN and Transitions Between Disease Stages}\label{sec:key_factor_2}
The estimated trajectory from GPA offers us a chance to measure the evolution of the disease via a time-series way. To sufficiently capture the relationship between the severity of disease and graph spatial (from the node connectivity) temporal (graph time series) features, we deploy recent works on temporal graph neural networks, e.g., Temporal Graph Convolutional
Network (TGCN) \citep{zhao2019t}. Temporal GNNs such as TGCN, TGAT \citep{xu2020inductive} and their successors, the so-called spatial-temporal GNNs (STGNNs) such as STGCN \citep{yu2017spatio} and DCRNN \citep{li2017diffusion} have been widely applied to many forecasting fields such as traffic flow/speed and time series. Specifically, assuming the time delay (i.e., lag) as 1, then our task is to find 
\begin{align}
    \widehat{\mathbf y}(t) = {g}_{\boldsymbol{\theta}}(\mathbf A(t), \mathbf X(t), \mathbf A(t-1), \mathbf X(t-1)),
\end{align}
where we let $g_{\boldsymbol{\theta}}$ be any temporal GCN model that approximates the ground truth disease severity $\mathbf y(t)$ at time $t$ using $\widehat{\mathbf y}(t)$ by leveraging time-dependent adjacency and feature information (i.e., $\mathbf A(t), \mathbf A(t-1), \mathbf X(t), \mathbf X(t-1)$). We note that here we assign the GPA estimated sequences of the graph as the index denote as $(t)$ (e.g., $\mathbf X(t)$) instead of keep using $i$ (e.g., $\mathbf X(i)$) which standards of graph $i$ without sequence (same as $\mathbf y_i$ to $\mathbf y(t)$). Adopted from the TGCN, we start with two layers of GCN \citep{kipf2016semi} followed by the so-called graph recurrent unit (GRU) and the model output (i.e., $\widehat{\mathbf y}$) is obtained by a single layer MLP. That is
\begin{align}\label{eq:tgcn}
    &\mathbf H(t) = \mathbf A(t) \mathbf X(t) \mathbf W, \quad
    \widetilde{\mathbf H}(t) = \mathrm{GRU} (\mathbf H(t)), \quad \widehat {\mathbf y}(t) = \mathrm{MLP}(\widetilde{\mathbf H}(t)),
\end{align}
in which we let $\mathbf H$ be the hidden state of the graph nodes and $\mathbf W$ learnable weights in GCN. Different from the original TGCN, where all graphs have the same structure, here, since we let the distance between nodes as the edge weights, we dynamically input $\mathbf A$ in our revised TGCN, and the model is trained via the MSE loss between $\mathbf y(t)$ and $\widehat{\mathbf y}(t)$.

\paragraph{Clustering and Transaction Between Disease Stages}
In order to segment disease into different stages, we provide a Gaussian Mixture Model (GMM) to the low-dimensional representations of the graphs (i.e., $\mathbf z$), as illustrated by Figure.~\ref{fig:clustering}, in which clusters are highlighted with different colors. We define the graph pairs $(\mathcal G(t), \mathcal G(t+1))$ as the transaction pairs between disease stages if $\mathcal G(t) \in \mathcal S_1$ and $\mathcal G(t+1) \in \mathcal S_2$, where $\mathcal S_1$ and $\mathcal S_2$ are the sets (clusters) that contain all the graphs in stage one and two, estimated by GMM. For example, we will have 3 pairs of graphs for four stages of disease in Figure.~\ref{fig:clustering}. To identify the most influential node that drives the changes of the disease stages, one can simply compute $\left \|\frac{\partial \widehat{\mathbf y}(t)}{\partial \mathbf x_s(t)}\right \|$ similar to Equation~(\ref{eq:sensitivity}). The difference is in this case, one only needs to select all \textbf{the first graph in each stage}, as based on the formulation of TGCN above, the disease severity (i.e., $\widehat{\mathbf y}(t)$) of that graph depends on the adjacency and feature information from both this graph and its previous step. Accordingly, computing $\left \|\frac{\partial \widehat{\mathbf y}(t)}{\partial \mathbf x_s(t)}\right \|$ directly illustrates how sensitive the current disease severity to the previous graph features. 

\section{Neural Stochastic Differential Equations and Other Downstream Tasks}\label{sec:nsde}
With the advance of leveraging GPA to provide pseudotime orders to the graph, we are now interested in analyzing the feature-changing dynamics along with the estimated trajectories. We note that given different starting and ending point settings from the original JR5558 dataset, we obtained three trajectories. Based on our inductive assumption that the transaction between graph pathway features (i.e., $\mathbf x$) is governed by stochastic differential equations (SDE) and can be approximated by the neural SDE method such that
\begin{align}\label{eq:sde}
    \frac{\partial \mathbf x(t)}{\partial t} = \psi_{\boldsymbol{\theta}} (\mathbf x(t), t) dt + \xi_{\boldsymbol{\phi}}(\mathbf x(t), t)d\mathbf B(t), 
\end{align}
where we let $\psi_{\boldsymbol{\theta}}$ and $ \xi_{\boldsymbol{\phi}}$ be drifting and diffusion terms which measure the deterministic and stochastic changes of the node features, approximated by two neural networks which parametrized by $\boldsymbol{\theta}$ and $\boldsymbol{\phi}$, respectively, and $d\mathbf B$ is the standard Wiener process. Accordingly, each node $\mathbf x$ owns 23 (number of graphs) steps to reach its destination, and based on the graph forming paradigm, we have 343 nodes in each graph.  Then we train a model to learn $\psi_{\boldsymbol{\theta}}$ and $ \xi_{\boldsymbol{\phi}}$ via torch SDE package \citep{li2020scalable}. 

\paragraph{Downstream Task One: Pathway Stability}
The construction of neural SDE will enable us to investigate the pathway stability (PS). Specifically, we are interested in the following: 
\begin{tcolorbox}[colback=cyan!5, breakable]
{\begin{center}\textit{ Which pathways remain stable throughout disease progression, and which pathways exhibit instability due to external perturbations?}
\end{center}
}
\end{tcolorbox}
Understanding this question is fundamental to uncovering the mechanisms underlying disease progression. Stable pathways often represent core regulatory circuits that drive disease dynamics, while unstable pathways are highly responsive to external stimuli and may serve as key intervention targets. Identifying these pathways can provide critical insights into therapeutic strategies, particularly for diseases characterized by stochastic transitions \citep{gupta2011stochastic,panegyres2022stochasticity}. Numerically, we quantify the PS by analyzing their diffusion terms time variances, that is
\begin{align}\label{eq:ps}
    \mathrm{PS} (\mathbf x) = \frac1T \sum_{t=1}^T\|\xi_{\phi}(\mathbf{x}(t+1), t+1) - \xi_{\phi}(\mathbf{x}(t), t)\|^2,
\end{align}
where $T = 23$ is the total step of the pathway trajectory. The formulation of PS follows directly from classic stochastic stability theory \citep{khasminskii2012stochastic}, where the diffusion term's behavior dictates long-term stability. As shown in the following Lemma. 

\begin{lem}[informal]\label{lem:stability}
    Let $\mathbf x(t) \in \mathbb R^d$ be a stochastic process governed by the Itô stochastic differential equation (SDE) given in Equation~(\ref{eq:sde}). Assuming
    the diffusion term \( \xi_{\phi}(\mathbf{x}(t), t) \) is continuous and differentiable and has bounded partial derivatives, 
    then under mild conditions, the mean-square stability of the system satisfies if $\exists C > 0, \sup_{t \geq 0} \mathbb{E} \|\xi_{\phi}(\mathbf{x}(t), t)\|^2 > C$, i.e., the system exhibits \textbf{mean-square divergence}, thus \textbf{unstable} ($\mathbb E \|\mathbf x(t)\| \rightarrow \infty$). On the other hand, if $\lim_{t \to \infty} \mathbb{E} \|\xi_{\phi}(\mathbf{x}(t), t)\|^2 \to 0$, then the system is \textbf{mean-square asymptotically stable}, i.e., $\mathbb E \|\mathbf x(t)\|^2$ is bounded. 
    \end{lem}

Full version of Lemma \ref{lem:stability} is in Appendix \ref{append:proof}. As shown in Lemma~\ref{lem:stability}, if the expected magnitude of $\xi_{\phi}(\mathbf{x}(t), t)$ remains bounded, the system is mean-square stable; otherwise, it diverges. Thus, PS serves as an empirical metric, capturing instability through the variance of $\xi_{\phi}(\mathbf{x}(t), t)$.
We show the proof in Appendix \ref{append:proof}, where we adapted the proof from the work in \citep{khasminskii2012stochastic} and the empirical results via Section \ref{sec:experiment}.

\paragraph{Downstream Task Two: Disease Bifurcation Points}
Apart from the stability analysis of the estimated SDE, we are also interested in determining whether a \textit{bifurcation point} (BP) aka. \textit{point of no return} exists in the progression of disease at the pathway level, as this potentially leads to irreversible disease progression \citep{flores2023bifurcation}. Specifically, we present the following definition. 
\begin{defn}[Point of No Return]\label{defn:BP}
    Let $\mathbf x(t) \in \mathbb R^d$ be a stochastic process governed by the following Itô stochastic differential equation (SDE) given in Equation~(\ref{eq:sde}). Define the system stability potential function as 
    \begin{align}
        J(\mathbf x, t) = - \int \psi_{\boldsymbol{\theta}} (\mathbf x, t) d\mathbf x 
    \end{align}
   We define the point of no return (bifurcation time) as a time $t^*$
  at which the pathway undergoes an irreversible transition in its dynamical behavior. Specifically, it is the first time at which at least one of the following conditions holds: (1): \textit{Loss of Stability Region}, that is $\nabla J(\mathbf x, t^*) \approx 0$ and $\frac{d}{dt}J(\mathbf x, t^*) >0$; or (2): \textit{Transition to a New Steady State}: $\lim_{t \rightarrow \infty} \mathbb E\|\mathbf x(t) - \mathbf x(t^*)\| > C, C>0$; (3): \textit{Diffusion Variance Explode}: $\sup_{t \geq t^*} \mathrm{Var}(\mathrm{PS}(\mathbf x, t)) > \delta$, where $\delta$ is a predefined threshold.   
\end{defn}
One can find that condition (1) checks whether the system reaches the critical point such that deterministic forces vanish $\nabla J(\mathbf x, t^*) \approx 0$, but the stability is deteriorating as $\frac{d}{dt}J(\mathbf x, t^*) >0$. Similarly, condition (2) suggests the changes of $\mathbf x$ will no longer go back to the velocity at time $t^*$, meaning the pathway undergoes an irreversible transition into a new state. Finally, condition (3) captures a sudden increase in stochastic fluctuations, signifying a transition into a highly unstable regime. Although the quantity of $C$ and $\delta$ vary between datasets, in our implementation, we simply set $C =1$ and $\delta = 0.1$, and we report our results in Section \ref{sec:experiment}. We further note that, in practice, the disease is detected mostly after its starting point (i.e., $t = 0$). Thus, we omit the trivial bifurcation point obviously at $t= 0$. We finally remark that in Appendix \ref{append:sde}, we extend our theory for the case when pathways interact with each other, yielding some additional discoveries for measuring the pathway stability and disease bifurcation points.


\section{Experimental Results}\label{sec:experiment}
In this section, we present the results of the numerical experiment and their biological interpretations. All experiments are conducted via one NVIDIA 4090 GPU with 24GB memory, and the implementation details, as well as additional discussions, are included in Appendix \ref{append:experiment}. Table~\ref{tab:pathways} summarizes our findings. 

\begin{table}[t]
\centering
\renewcommand{\arraystretch}{5} %
\caption{Summary of the findings, including top 5 SPs, SPs for disease stage changes, and top 5 stable and non-stable pathways from the results of pathway stabilities.}
\label{tab:pathways}
\resizebox{1\linewidth}{!}{ %
\begin{tabular}{cccccc}
\toprule
\rowcolor[HTML]{C0C0C0} 
\textbf{\LARGE Tasks}             & \multicolumn{5}{c}{\cellcolor[HTML]{C0C0C0}\textbf{ \LARGE Description of the Top 5 Pathways (Left to Right) }} \\ 
\midrule
\textbf{\LARGE Sensitive Pathways (SP)}    & \Large Caffeine metabolism & \Large Primary immunodeficiency & \Large Small cell lung cancer & \Large Long-term potentiation & \Large Proteoglycans in cancer \\ 
\midrule
\textbf{\LARGE SPs for Stage (0-1)} & \Large Biotin metabolism & \Large Caffeine metabolism & \Large Virion \Large Human immunodeficiency virus & \Large Virion Flavivirus & \Large Virion Adenovirus \\ 
\textbf{\LARGE SPs for Stage (1-2)} & \Large alpha-Linolenic acid metabolism & \Large Caffeine metabolism & \Large Virion Human immunodeficiency virus & \Large Virion Flavivirus & \Large Axon guidance \\ 
\textbf{\LARGE SPs for Stage (2-3)} & \Large Caffeine metabolism & \Large alpha-Linolenic acid metabolism & \Large Virion Human immunodeficiency virus & \Large Biotin metabolism & \Large Virion Adenovirus \\ 
\midrule
\textbf{\LARGE Stable Pathways}            & \Large FoxO signaling pathway & \Large Transcriptional misregulation in cancer & \Large Peroxisome & \Large PPAR signaling pathway & \Large Biosynthesis of unsaturated fatty acids \\ 
\textbf{\LARGE Non-stable Pathways}        & \Large Caffeine metabolism & \Large Virion Human immunodeficiency virus & \Large Virion Adenovirus & \Large Phototransduction & \Large Virion Flavivirus \\ 
\bottomrule
\end{tabular}
}
\end{table}

\subsection{Biology Discussion of the Findings}
While KEGG-defined pathways (Section~\ref{sec:graph_forming}) provided a useful starting point, our analysis revealed unexpected pathway influences, emphasizing the need for recontextualization within the specific tissue environment of the retina. One key pathway cluster, node 79, was initially classified under "Caffeine Metabolism" in KEGG annotations. However, our analysis revealed that its constituent genes primarily function in metabolic detoxification and oxidative stress regulation. These processes form a protective metabolic network that mitigates oxidative damage and toxic accumulation in retinal cells. Genes involved in xenobiotic metabolism and purine breakdown play a critical role in maintaining photoreceptor stability by reducing harmful by-products. Whether these pathways act as primary drivers of retinal degeneration or as secondary adaptive responses remains an open question. However, their high sensitivity to disease progression suggests that targeting metabolic detoxification pathways could be a novel therapeutic strategy.

In addition, another key driver of disease stage transition, from Stage 0 to Stage 1, was classified under KEGG as "Virion Human Immunodeficiency Virus." Rather than viral-related functions, this pathway consists of genes involved in immune regulation, chemokine-mediated signaling, and antigen recognition. Within the retina, an immunoprivileged tissue, this network plays a crucial role in maintaining immune homeostasis. Key molecular players include CCR5 and CXCR4, which regulate T cell infiltration into the retina; the CD209 family, which monitors and responds to molecular threats; and CD4, which coordinates immune tolerance and response balance. If this immune surveillance network is overactivated, the resulting inflammatory response could drive retinal damage, accelerating degeneration. This suggests that modulating immune checkpoint pathways could be an effective approach to delaying retinal fibrosis. These findings underscore the importance of refining pathway annotations to reflect actual biological functions in the retina.

Furthermore, a key outcome of our study was the identification of a bifurcation point at Step 4 of our estimated trajectory, which represents an irreversible shift in the transcriptomic landscape. Notably, \textbf{all major pathways} exhibited their own bifurcation at this step by satisfying our proposed condition (2) (i.e., transition to a new steady state), suggesting a system-wide molecular transition that acts as a "point of no return". This finding suggests that the most effective therapeutic window is during the early stages of retinal degeneration, before this critical transition occurs. These findings are supported by clinical observations, further validating our computational approach.






\section{Limitations of the Study and Future Directions}
Despite our novel modeling methods and results, which serve as the strength of this study, our study has some limitations. Due to dataset constraints, our current model focused on pathway-level dynamics rather than individual gene-level interactions. Future work should refine this approach at the gene level, potentially identifying direct genetic targets for intervention. In addition, in Appendix \ref{append:sde} we show a generalization of our model to measure the situation when pathways can interact with each other. Additionally, pathway annotations in KEGG are often too broad or non-specific for tissue- and disease-specific contexts, highlighting the need for retina-specific pathway ontology development.

Beyond macular degeneration, our approach could be applied to other complex diseases where molecular interactions drive progressive phenotypic changes. Future work should extend this methodology by conducting gene-level analysis within top-ranked pathways to identify direct therapeutic targets, refining pathway interpretations through advanced computational tools, and integrating multi-omics data for a more comprehensive disease model. Our study highlights the potential of computational systems biology in uncovering actionable disease insights, paving the way for precision medicine approaches in retinal disease treatment.

From the machine learning perspective, we shall consider the interactions between both genes and their pathways. Ideally, it is preferable to construct a so-called hierarchical graph with two levels, with a gene graph as level one and a pathway graph as level two. A machine learning model is expected to be built to identify the relationship between the hierarchical graph and disease severity after the collection between different levels of graphs is built. Naturally, this modification will lead to a more complex but higher biological interpretable SDE and the formulations of the downstream tasks, we leave this promising direction as our future work.

\bibliographystyle{plainnat}
\bibliography{ref}

\appendix
\section{Related Works}\label{append:related_works}

\paragraph{Graph Neural Networks and Different Types of Graphs}
GNNs were originally proposed to resolve the challenge of data point dependencies via the traditional convolution neural networks, which, in general, treat every input data point independently of each other \citep{kipf2016semi,defferrard2016convolutional}. By considering the so-called adjacency information stored in the graph, GNNs propagate graph node features by aggregating its neighboring information \citep{wu2020comprehensive}. Such propagation paradigm has made GNNs one of the most successful tools in generating predictions via various types of graphs, such as citation networks \citep{wu2020comprehensive}, social networks \citep{sharma2024survey}, molecules (as well as protein and ligands) \citep{zhang2022graph}, traffic networks \citep{jiang2022graph}, to name a few. These graphs vary from different levels of measurements of the data, e.g., scientific papers in citation networks \citep{yang2016revisiting}, atoms, and amino acids in protein and ligand graphs \citep{wu2018moleculenet,GilmerSchoenholzRiley2017} and nodes are connected with different types of attributes (e.g., citations and chemical bonds). In this work, we provide a \textbf{novel graph dataset} (known as JR5558, more details see Section \ref{sec:preliminaries}) in which graphs are formed by the genetic pathways (as nodes) and pathway similarities (as edges), serving as the profiles of the experimental objects (e.g., mice). In addition, we also label these graphs with mice's lesion severity scores (as a continuous variable) obtained from fundus photographs of these mice, where severity was quantified by measuring subretinal lesion size.We train two different types of GNNs to capture the patterns between genetic pathways and mice lesion severity scores: one that does not consider temporal information, and another designed to account for temporal information from the estimated disease progression in the mice.

\paragraph{Pseudotime Analysis and Stochastic Differential Equations}
Pseudotime analysis (PA) was originally developed in single-cell transcriptomics to reconstruct cell differentiation trajectories from static snapshots of gene expression profiles \citep{trapnell2014dynamics,haghverdi2016diffusion,wei2021dtflow}. Since time-resolved measurements of individual cells are often infeasible, pseudotime methods infer an intrinsic ordering (e.g., trajectories) of cells based on their transcriptional similarities, providing insights into dynamic biological processes. While it is a powerful tool widely applied in biology and medical science, its adoption in the machine learning community has only gained significant attention in recent years. For example, recent work has explored using diffusion models to infer pseudotime from single-cell transcriptomic data \citep{mosspseudotime}. In this work, we apply PA to the embeddings of graphs constructed from the genetic pathways of the mice and analyze the estimated trajectories via neural stochastic differential equations \citep{li2020scalable} (see Section \ref{sec:nsde} for more details), which allow us to further exploit pathway stability and disease bifurcation points. This paradigm paves the path of incorporating advanced machine learning approaches to the exploration of the fundamental problems in complex biological systems.


\section{Proofs}\label{append:proof}

\begin{lem}[Full Verison of Lemma \ref{lem:stability}]
    Let $\mathbf x(t) \in \mathbb R^d$ be a stochastic process governed by the Itô stochastic differential equation (SDE) given in Equation~(\ref{eq:sde}). 
    Assuming there exist constants \( C_1, C_2 > 0 \) such that the ratio between the squared drift term and the trace of the diffusion term in Eq~(\ref{eq:sde}) remains bounded:
\begin{equation}\label{eq:same_scale}
    C_1 \leq \frac{\|\psi_{\theta}(\mathbf{x}, t)\|^2}{\mathrm{Tr}(\xi_{\phi}(\mathbf{x}, t) \xi_{\phi}^\top(\mathbf{x}, t))} \leq C_2, \quad \forall \mathbf{x}, t.
\end{equation}
    Further assuming
    the diffusion term \( \xi_{\phi}(\mathbf{x}(t), t) \) is continuous and differentiable and has bounded partial derivatives, 
    then the mean-square stability of the system satisfies if $\exists C > 0, \sup_{t \geq 0} \mathbb{E} \|\xi_{\phi}(\mathbf{x}(t), t)\|^2 > C$, i.e., the system exhibits \textbf{mean-square divergence}, thus \textbf{unstable} ($\mathbb E \|\mathbf x(t)\| \rightarrow \infty$). On the other hand, if $\lim_{t \to \infty} \mathbb{E} \|\xi_{\phi}(\mathbf{x}(t), t)\|^2 \to 0$, then the system is \textbf{mean-square asymptotically stable}, i.e., $\mathbb E \|\mathbf x(t)\|^2$ is bounded. 
\end{lem}

\begin{proof}
Our conclusion is adopted from the work \cite{khasminskii2012stochastic} in which the stability of SDE is analyzed in a more solid manner. We start by discretizing the SDE with a small step \( \Delta t \), we obtain:
\begin{equation}
    \mathbf{x}(t+\Delta t) = \mathbf{x}(t) + \psi_{\theta}(\mathbf{x}(t), t) \Delta t + \xi_{\phi}(\mathbf{x}(t), t) \Delta \mathbf{B}(t).
\end{equation}
Taking the expectation of the squared difference, we get:
\begin{equation}
    \mathbb{E} \|\mathbf{x}(t+\Delta t) - \mathbf{x}(t)\|^2 = \mathbb{E} \|\psi_{\theta}(\mathbf{x}(t), t) \Delta t + \xi_{\phi}(\mathbf{x}_t, t) \Delta \mathbf{B}(t)\|^2.
\end{equation}
By independence of the drift and diffusion terms:
\begin{equation}
     \mathbb{E} \|\mathbf{x}({t+\Delta t}) - \mathbf{x}(t)\|^2= \mathbb{E} \|\psi_{\theta}(\mathbf{x}(t), t) \Delta t\|^2 + \mathbb{E} \|\xi_{\phi}(\mathbf{x}(t), t) \Delta \mathbf{B}(t)\|^2.
\end{equation}
Then, we can denote the following.
\begin{equation}
    \mathbb{E} \|\psi_{\theta}(\mathbf{x}(t), t) \Delta t\|^2 = \|\psi_{\theta}(\mathbf{x}(t), t)\|^2 \Delta t^2, \tag{Drift term contribution} 
\end{equation}
and 
\begin{equation}
    \mathbb{E} \|\xi_{\phi}(\mathbf{x}(t), t) \Delta \mathbf{B}(t)\|^2 = \mathbb{E} \left[ \text{Tr}(\xi_{\phi}(\mathbf{x}(t), t) \xi_{\phi}(\mathbf{x}(t), t)^\top) \right] \Delta t. \tag{Diffusion term contribution} 
\end{equation}

Thus, for sufficiently small \( \Delta t \), the dominant term is:
\begin{equation}
    \mathbb{E} \|\mathbf{x}{(t+\Delta t)} - \mathbf{x}(t)\|^2 \approx \text{Tr}(\xi_{\phi}(\mathbf{x}(t), t) \xi_{\phi}(\mathbf{x}(t), t)^\top) \Delta t.
\end{equation}
This is because for standard Wiener process, i.e., $\mathbf B(t)$ we have $\mathbb \|\Delta \mathbf B(t) \|^2 = \Delta t$, therefore, the rate of the change of the diffusion time is proportional to $\Delta t$. One can also check the rate of change of the drifting term is proportional to $\Delta t^2$. This suggests once the steps of the disease evolution are dense enough, i.e., for any small step $\Delta t$ there is an observation $\mathbf x(t+\Delta t)$. \textbf{One only needs to investigate the diffusion term of the SDE}. Specifically, 
from stochastic stability theory \citep{khasminskii2012stochastic} (Theorem 6.13), if \( \sup_{t \geq 0} \mathbb{E} \|\xi_{\phi}(\mathbf{x}(t), t)\|^2 > C \), then \( \mathbb{E} \|\mathbf{x}(t)\|^2 \to \infty \), implying instability. If \( \lim_{t \to \infty} \mathbb{E} \|\xi_{\phi}(\mathbf{x}(t), t)\|^2 \to 0 \), then \( \mathbb{E} \|\mathbf{x}(t)\|^2 \) remains bounded, implying stability, and this completes the graph. 
\end{proof}

\begin{rem}
Compared to the informal version of Lemma \ref{lem:stability}. The additional assumption is commonly observed in physical and biological systems, such as Langevin dynamics and gene regulatory networks \citep{elowitz2000synthetic}. Additionally, prior works in neural stochastic differential equations \citep{kidger2021neural} enforce similar regularization strategies to maintain numerical stability.
\end{rem}

\section{Experiment Details and Additional Results}\label{append:experiment}

\subsection{Identification of SPs}
In this section, we show details regarding the task of finding the SPs of the disease severity. In addition to the graph forming method, we illustrate in Section \ref{sec:graph_forming}. We note that after the graph node features are constructed, the edge weight of the graph is computed via the distance of the node features from a Gaussian kernel. That is
\begin{align}
    \mathbf W_{i,j} = \mathrm{e}^{(-\|\mathbf x_i - \mathbf x_j\|^2)},
\end{align}
in which we intrinsically let the variance of the Gaussian kernel as 1. 

\paragraph{Disease Sensitivity Without GPA}
In terms of the model training of the first task mentioned in Section \ref{sec:key_factor_1}. We first did the normalization of the graph adjacency matrix that is $\widehat{\mathbf A} = \mathbf D^\frac12 (\mathbf A + \mathbf I) \mathbf D^\frac12$ by also adding the self-loop for each node. We split our dataset with 70\% for training and 20\% for testing. We deploy the GCN model together with global mean pooling with one additional MLP layer to fit the node features to the disease severity. That is
\begin{align}
    \mathbf X (\ell +1 ) = \mathrm{GCN} (\mathbf X(\ell), \widehat{\mathbf A}), \quad \widehat{\mathbf y} = \mathrm{MLP}(\mathrm{MeanPool}(\mathbf X(\ell +1)),
\end{align}
and we let $\ell$ be the number of layers, which is set as 2 in our modeling. For the hyperparameters, we let the hidden dimension of the GCN model as 64 and the dropout ratio as 0.5 with the learning rate as $1\mathrm{e}^{-3}$ and weight decay as $1\mathrm{e}^{-4}$. The model is trained with 10 runs in which every run owns 200 epochs. As we have illustrated in Section \ref{sec:key_factor_1}, the sensitivity of one pathway node $s$ is then computed through $\mathrm{Sensitivity}(\mathbf x_s) = \sum_{i=1}^m\left\|\frac{\partial \widehat {\mathbf y}_i}{\partial \mathbf x_s}\right \|$ by directly calling the gradient from the loss backward process once the model is trained. 

\paragraph{GPA and Temporal GCN}
For the GPA, after we obtained the latent representations of the graphs, i.e., $\widehat{\mathbf z}$. We did unsupervised clustering applying the Gaussian Mixture Model (GMM), followed by the KNN algorithm with $K = 10$ to sufficiently engage the connectivity between nodes. In terms of the MST and estimated trajectories, we fixed two starting points, which own zero severity, and ending points, which have the most server disease conditions, yielding four trajectories. To measure the temporal relationships between graphs, we deploy TGCN \citep{zhao2019t} with specific feature propagation additional to Equation~(\ref{eq:tgcn}). That is
\begin{align}
    &\mathbf H(t) = \mathbf A(t) \mathbf X(t) \mathbf W, \\
    &\mathbf Q(t) = \mathbf A(t-1) \mathbf H(t-1) \mathbf W_q + \mathbf X(t)\mathbf P_q, \\
    &\mathbf R(t) = \mathbf A(t-1) \mathbf H(t-1)\mathbf W_r  +  \mathbf X(t)\mathbf P_r, \\
    & \widetilde{\mathbf H}(t) = \mathrm{tanh} \left (\mathbf A(t-1) (\mathbf R(t) \odot \mathbf H(t-1) \mathbf W_h) +  \mathbf H(t) \mathbf P_h \right),\\
    & \mathbf H(t) = (1-\mathbf Q(t)) \odot \mathbf H(t-1) + \mathbf Q(t) \odot \widetilde{\mathbf H}(t),\\   
    & \widehat {\mathbf y}(t) = \mathrm{MLP}({\mathbf H}(t)),
\end{align}
where all the $\mathbf W$ and $\mathbf P$ are the learnable matrices. In our practice implementation, the data is split using torch static graph temporal signal split with 70\% of training and the rest of the validation and testing. We let the hidden dimension be 96, the learning rate and weight decay be 0.01, and 1$\mathrm{e}^{-3}$, respectively. Similarly, the model is trained 10 times with 200 epochs within each time.  

\subsection{More on Neural SDEs}\label{append:sde}
Apart from the neural SDE models that we leveraged in Section \ref{sec:nsde}, one may interested in the SDEs that measure the situation when pathway interaction matters. Recall that in Section \ref{sec:graph_forming}, we developed pathway profile graphs for each object to sufficiently reflect its status, and this indicates us to \textbf{consider the interactions between pathways through the connectivities between them}. Accordingly, to measure the pathway interactions, we show the revised SDE as follows. 
\begin{align}\label{eq:sde_gnn}
    \frac{\partial \mathbf x(t)}{\partial t} = \big(\psi_{\boldsymbol{\theta}} (\mathbf x(t), t) + \sum_{j\sim i}\mathbf A_{i,j} \cdot \phi_\gamma (\mathbf x_j(t),t )\big)dt + \xi_{\boldsymbol{\phi}}(\mathbf x(t), t)d\mathbf B(t), 
\end{align}
where $\sum_{j\sim i}\mathbf A_{i,j} \cdot \phi_\gamma (\mathbf x_j(t),t )$ is the standard graph spatial convolution between each node and its neighbors. In practice, $\phi_{\boldsymbol{\gamma}} (\mathbf x_j(t),t )$ simply denotes a reweighting paradigm based on the inductive bias of the graph and can be optionally omitted. The formulation in Equation~(\ref{eq:sde_gnn}) directly shows that graph convolution joins the predefined drifting terms to determine the feature changes from multiple facets and form the easiest case, the newly defined drifting term, i.e., $\psi_{\boldsymbol{\theta}} (\mathbf x(t), t) + \sum_{j\sim i}\mathbf A_{i,j} \cdot \phi_\gamma (\mathbf x_j(t),t )$, in practice, can be even implemented by a MLP plus one GNN. In fact, it is well-known that one can denote Equation~(\ref{eq:sde_gnn}) as
\begin{align}
    \frac{\partial \mathbf x(t)}{\partial t} = \big(\psi_{\boldsymbol{\theta}} (\mathbf x(t), t) + \sum_{j\sim i}\mathbf A_{i,j} \nabla\mathbf x_i\big)dt + \xi_{\boldsymbol{\phi}}(\mathbf x(t), t)d\mathbf B(t),
\end{align}
where we denote graph gradient $\nabla \mathbf x_i = \mathbf x_j - \mathbf x_i$. As in most cases, we don't have a similar assumption in Equation~(\ref{eq:same_scale}), thus to evaluate the pathway stability, one shall also count the contribution of $\nabla \mathbf x_i = \mathbf x_j - \mathbf x_i$. Accordingly, we can extend Equation~(\ref{eq:ps}) to obtain graph pathway stability (GPS) as  
\begin{align}
     \mathrm{GPS} (\mathbf x) = \frac1T \sum_{t=1}^T\|\xi_{\phi}(\mathbf{x}(t+1), t+1) - \xi_{\phi}(\mathbf{x}(t), t)\|^2 + \lambda \sum_{i\sim j} \mathbf A_{i,j} (\mathbf x_j(t) - \mathbf x_i(t))^2,
\end{align}
where $\sum_{i\sim j} \mathbf A_{i,j} (\mathbf x_j(t) - \mathbf x_i(t))^2$ is the well-known Dirichlet energy of the node features, which measures the total variation of the node features. In addition to the conclusion of Lemma \eqref{lem:stability}, the inclusion of $\sum_{i\sim j} \mathbf A_{i,j} (\mathbf x_j(t) - \mathbf x_i(t))^2$ brings additional uncertainty to the system's stability. We then have the following conclusion.
\begin{cor}[GPS]\label{cor:gps}
    If $\mathbf A_{i,j} \geq 0, \forall i,j \in \mathcal E$, with the conditions obtained from Lemma \ref{lem:stability}, the system defined in Equation~(\ref{eq:sde}) is \textbf{mean-square asymptotically stable}; otherwise, it is\textbf{ unstable}. 
\end{cor}
\begin{proof}
    The proof is straightforward since if $\mathbf A_{i,j} \geq 0$, then asymptotically, $\mathbf A\mathbf X$ term will keep homogenizing every node feature accordingly to its neighbors, causing $\mathbf x_j = \mathbf x_i, t \rightarrow \infty$. Therefore, the pathway stability can still be evaluated using the $\mathrm{PS}(\mathbf x)$. However, if one has $\mathbf A_{i,j} \leq 0$, then the system will have Dirichlet energy explode, assuming in the continuous case, $\mathbf x$ is a function with infinite dimensionality. Therefore $\mathrm{GPS} (\mathbf x) \rightarrow \infty$, resulting an unstable system. 
\end{proof}

Similarly, for the disease bifurcation point (BP), we can revise Definition~\ref{defn:BP} when we take the potential of node interactions into account. We provide the following definition. 

\begin{defn}[Graph-Interacted Point of No Return] Let $\mathbf{x}(t) \in \mathbb{R}^d$ be a stochastic process
define in Equation~(\ref{eq:sde_gnn}). Define the system stability potential function, incorporating both self-dynamics and neighborhood influence, as 
\begin{align} J(\mathbf{x}_i, t) = - \int_0^t \Big[ \psi_{\boldsymbol{\theta}} (\mathbf{x}_i, t') + \sum_{j\sim i} \mathbf A_{ij}  (\mathbf{x}_j, t') \Big] dt'. 
\end{align} 
We define the \textbf{point of no return} (bifurcation time) as a time $t^*$ at which the pathway undergoes an irreversible transition in its dynamical behavior. Specifically, it is the first time at which at least one of the following conditions holds:
(1): Loss of Stability Region: The pathway leaves a local stability region, $\nabla J(\mathbf{x}_i, t^*) \approx 0, \quad \text{and} \quad \frac{d}{dt} J(\mathbf{x}_i, t^*) > 0$; (2): Transition to a New Steady State: The system irreversibly shifts to a new equilibrium: $\lim_{t \rightarrow \infty} \mathbb{E} \|\mathbf{x}_i(t) - \mathbf{x}_i(t^*)\|^2 > C,  C > 0$; (3): Diffusion Variance Explodes: The fluctuation in the pathway dynamics grows beyond a threshold, $\sup_{t \geq t^*} \mathrm{Var}(\mathrm{PS}(\mathbf{x}_i, t)) > \delta,  \delta > 0$; (4):
Strong Neighborhood Influence: the pathway graph owns at least one negative eigenvalue, i.e., energy explode as illustrated in Corollary \ref{cor:gps}.
\end{defn}

The above definition extends the previous definition of the bifurcation point by counting the contribution of the interactions between pathways.

\end{document}